\documentclass[letterpaper]{article} 
\usepackage{aaai24}  
\usepackage{times}  
\usepackage{helvet}  
\usepackage{courier}  
\usepackage[hyphens]{url}  
\usepackage{graphicx} 
\usepackage{natbib}  
\usepackage{caption} 
\usepackage{algorithm}
\usepackage{algorithmic}

\usepackage{amsmath}
\usepackage{amsthm}
\usepackage{amsfonts} 

\newtheorem{theorem}{Theorem}[section]

\usepackage{stmaryrd}
\usepackage{booktabs}
\usepackage{multirow}

\usepackage{newfloat}
\usepackage{listings}
\DeclareCaptionStyle{ruled}{labelfont=normalfont,labelsep=colon,strut=off} 
\floatstyle{ruled}
\newfloat{listing}{tb}{lst}{}
\floatname{listing}{Listing}
\title{Rethinking Dimensional Rationale in Graph Contrastive Learning\\ from Causal Perspective}
\author{
    Qirui Ji\textsuperscript{\rm 1 \rm 3} \equalcontrib,
    Jiangmeng Li\textsuperscript{\rm 1 \rm 2}  \equalcontrib\thanks{Corresponding author.},
    Jie Hu\textsuperscript{\rm 4},
    Rui Wang\textsuperscript{\rm 1 \rm 2 \rm 3},
    Changwen Zheng\textsuperscript{\rm 1 \rm 3},
    Fanjiang Xu\textsuperscript{\rm 1 \rm 3}
}
\affiliations{
    \textsuperscript{\rm 1}Science \& Technology on Integrated Information System Laboratory, Institute of Software Chinese Academy of Sciences\\
    \textsuperscript{\rm 2}State Key Laboratory of Intelligent Game\\
    \textsuperscript{\rm 3}University of Chinese Academy of Sciences\\
    \textsuperscript{\rm 4}State Key Laboratory of Computer Science, Institute of Software Chinese Academy of Sciences

    \{jiqirui2022, jiangmeng2019, wangrui, changwen, fanjiang\}@iscas.ac.cn, hujie@ios.ac.cn
}

\usepackage{bibentry}

\begin{document}

\maketitle

\begin{abstract}
Graph contrastive learning is a general learning paradigm excelling at capturing invariant information from diverse perturbations in graphs. Recent works focus on exploring the \textit{structural} rationale from graphs, thereby increasing the discriminability of the invariant information. However, such methods may incur in the mis-learning of graph models towards the interpretability of graphs, and thus the learned noisy and task-agnostic information interferes with the prediction of graphs. 
To this end, with the purpose of exploring the intrinsic rationale of graphs, we accordingly propose to capture the \textit{dimensional} rationale from graphs, which has not received sufficient attention in the literature. 
The conducted exploratory experiments attest to the feasibility of the aforementioned roadmap. To elucidate the innate mechanism behind the performance improvement arising from the dimensional rationale, we rethink the dimensional rationale in graph contrastive learning from a causal perspective and further formalize the causality among the variables in the pre-training stage to build the corresponding structural causal model. On the basis of the understanding of the structural causal model, we propose the dimensional rationale-aware graph contrastive learning approach, which introduces a learnable dimensional rationale acquiring network and a redundancy reduction constraint. The learnable dimensional rationale acquiring network is updated by leveraging a bi-level meta-learning technique, and the redundancy reduction constraint disentangles the redundant features through a decorrelation process during learning. Empirically, compared with state-of-the-art methods, our method can yield significant performance boosts on various benchmarks with respect to discriminability and transferability. The code implementation of our method is available at \url{https://github.com/ByronJi/DRGCL}.
\end{abstract}

\section{Introduction}


Graph contrastive learning (GCL) is a general learning paradigm excelling at seeking to understand invariant information from diverse perturbations in graphs \cite{graphcl, adgcl, joao, simgrace}. 
However, most of these methods focus on building sophisticated data augmentations for GCL, while the intrinsic interpretability in graph representations is not explored, such that the theoretical guarantee for the performance improvement arising from adopting such approaches is insufficient, and the model trained by following these methods may learn stochastic noisy and task-agnostic information, thereby confusing the prediction on downstream tasks. Therefore, the graph rationale exploration is provoked to understand the knowledge driving the model to make certain predictions \cite{dir}, where rationale is a specific subset of graph features, e.g., graph structure, which can guide or explain the model's predictions \cite{gnnexplainer}. 
Successes achieved by RGCL \cite{rgcl} demonstrate that exploring rationales in graphs can indeed promote the model to learn discriminative representations in GCL. 
RGCL focuses on exploring the \textit{structural} rationale (SR) from graphs, i.e., the structure containing specific edges or nodes that are correlated with the prediction of graphs. However, the features contained by the nodes or messages passing through the edges may still include certain discriminative information. Thus, arbitrarily removing or assigning weights to the graph structure can undermine the discriminability of the learned representation. Concretely, we raise a crucial question:

\textbf{\textit{``Does there exist a manner to explore the intrinsic rationale in graphs, thereby improving the GCL predictions?''}}

With the question in mind, we conduct exploratory experiments with GraphCL \cite{graphcl} on the biochemical molecule dataset PROTEINS, and the social network dataset REDDIT-BINARY (RDT-B), where we randomly preserve certain dimensions, i.e., a subset of the representations, while blocking the others. The experimental results are illustrated in Figure \ref{fig:motivation experiment}. We observe from the results and find that the graph representations only preserving specific dimensions indeed achieve better performance than the primitive representations, and such dimensions are treated as \textit{dimensional} rationales (DRs) for the graph. The exploratory experiments jointly prove the existence of DRs and the positive effects of specific DRs in the prediction of GCL. 
The intuition behind the experimental exploration is that compared with the SR, the DR is intrinsic to the representations learned by GCL methods, which can tackle the long-standing issue of the SR and further achieve the desideratum that jointly preserves the discriminative information and blocks the task-agnostic information of representations. We provide theoretical analysis to demonstrate that the proposed DR can degenerate into the conventional SR for GCL. 

\begin{figure*}[t]
    \centering	\includegraphics[width=\textwidth]{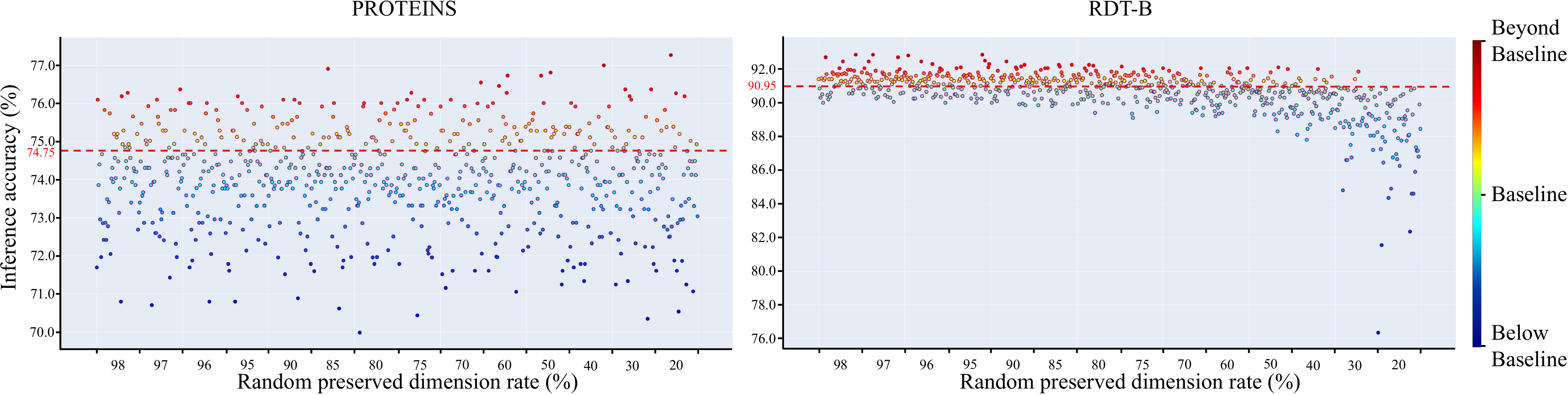}
	\caption{Experimental scatter diagrams obtained by GraphCL with randomly preserving dimensions on PROTEINS and RDT-B datasets. The red dashed lines denote the performance achieved by the primitive representation of GraphCL. The colored scattered points denote the downstream classification performance of embeddings with certain dimensions preserved. Note that the unreserved dimensions are directly valued by 0. The experimental principle emerges from the intuition that the prediction on downstream tasks may be significantly affected if the multi-dimensional representations are perturbed.}
	\label{fig:motivation experiment}
\end{figure*}

However, the innate mechanism of the performance improvement brought about by introducing the DR is not sufficiently explored, such that we rethink the DR of graphs from the causal perspective and accordingly develop a structural causal model (SCM). By understanding the SCM, we disclose a counter-intuitive conclusion: the acquired graph DR is determined as a causal confounder in GCL. The reason is that the principle of unsupervised learning incurs the inconsistent variation of the acquired graph DR, such that the acquired rationale may improve or degenerate the model performance on downstream tasks, and inspired by the theory of causal inference \cite{pearl2000models,glymour2016causal}, we assert that the acquired graph DR is a causal confounder. The theoretical analysis can further be proved by the empirical evidence. Accordingly, as shown in Figure \ref{fig:motivation experiment}, we observe that the points with different preserved dimension rates are scattered around the baseline dashes, which proves that the inconsistency of the graph DRs may improve or degenerate the model performance. 


To this end, we intuitively propose the \textit{\underline{D}imensional \underline{R}ationale-aware \underline{G}raph \underline{C}ontrastive \underline{L}earning}, namely \textit{DRGCL}, which initially acquires the graph DRs and further adopts the backdoor adjustment technique \cite{glymour2016causal}. Specifically, we introduce the learnable graph DR acquiring network, which is trained by adopting a bi-level meta-learning technique. To extend the representation space of the acquired DR, we apply the graph DR redundancy reduction as a regularization term during training. We provide solid theoretical analyses to prove the validity of DRGCL. Empirically, we compare our method with various baselines on benchmark graph datasets, which further demonstrates the effectiveness of DRGCL. \textbf{Contributions}:
\begin{itemize}
    \item We present a heuristic experiment to demonstrate the existence of the graph DR and further provide theoretical analysis to prove that compared with the conventional graph SR, the graph DR is more intrinsic to GCL.
    
    \item We formalize the mechanism of introducing DRs by building an SCM and demonstrate that the acquired DR is a causal confounder in GCL with sufficient theoretical and empirical evidence.
    
    \item We propose DRGCL to acquire redundancy-against DRs and perform the backdoor adjustment on SCM, thereby consistently improving GCL performance.
    
    \item We provide solid theoretical and experimental analyses, which jointly demonstrate the effectiveness of our method in terms of discriminability and transferability.
\end{itemize}

\section{Related Works}

\subsection{Graph Contrastive Learning}
Many methods have been used to study graph-level contrastive learning.
GraphCL \cite{graphcl} designs four types of general augmentations for GCL. 
ADGCL \cite{adgcl} optimizes adversarial graph augmentation strategies  to prevent Graph Neural Networks (GNNs) from capturing redundant information.
JOAO \cite{joao} selects augmentation pairs in GraphCL by an automated approach to solve the trial-and-errors. 
SimGRACE \cite{simgrace} utilizes an original GNN model and its perturbed version as encoders to obtain correlated views further avoiding the cost of trial-and-errors.
RGCL \cite{rgcl} uses GNNExplainer \cite{gnnexplainer} to find 
invariant SRs to dig discriminative information.   
Our method applies DRs to our models, which can find more intrinsic discriminability. 

\subsection{Graph Rationalization}
Rationalization in Graphs has two research directions: post-hoc explainability and intrinsic interpretability. Post-hoc explainability uses separate methods \cite{gnnexplainer,cf2explainer} to attribute predictions to the input graph. Intrinsic interpretability integrates a rationalization module, e.g., graph attention \cite{gat,dir} or graph pooling \cite{DBLP:conf/icml/LeeLK19,DBLP:conf/aaai/RanjanST20}, into the model. The rationalization module employs soft or hard masks on the input graph to guide the model's decisions. While existing methods use SRs from masked subgraphs to train the GNN, our method directly captures DRs within the graph embeddings.

\subsection{Causal Inference}
Causal inference \cite{pearl2000models,glymour2016causal} has been widely applied in computer science through deconfounding and counterfactual inference. Deconfounding methods \cite{disentangle,robust,DBLP:conf/icml/QiangLZ0X22,metaattention} estimate the direct causal effect behind confounders. Counterfactual inference \cite{DBLP:conf/www/TanGFGX0Z22} aims at finding the smallest change in the input which affects the
model’s prediction. Our work introduces the dimension confounder in GCL with an SCM. Guided by SCM, we utilize the backdoor adjustment to obtain the direct causal effect between the learned embedding and the predicted label.

\begin{figure}[t]
\begin{center}
\includegraphics[width=0.4\textwidth]{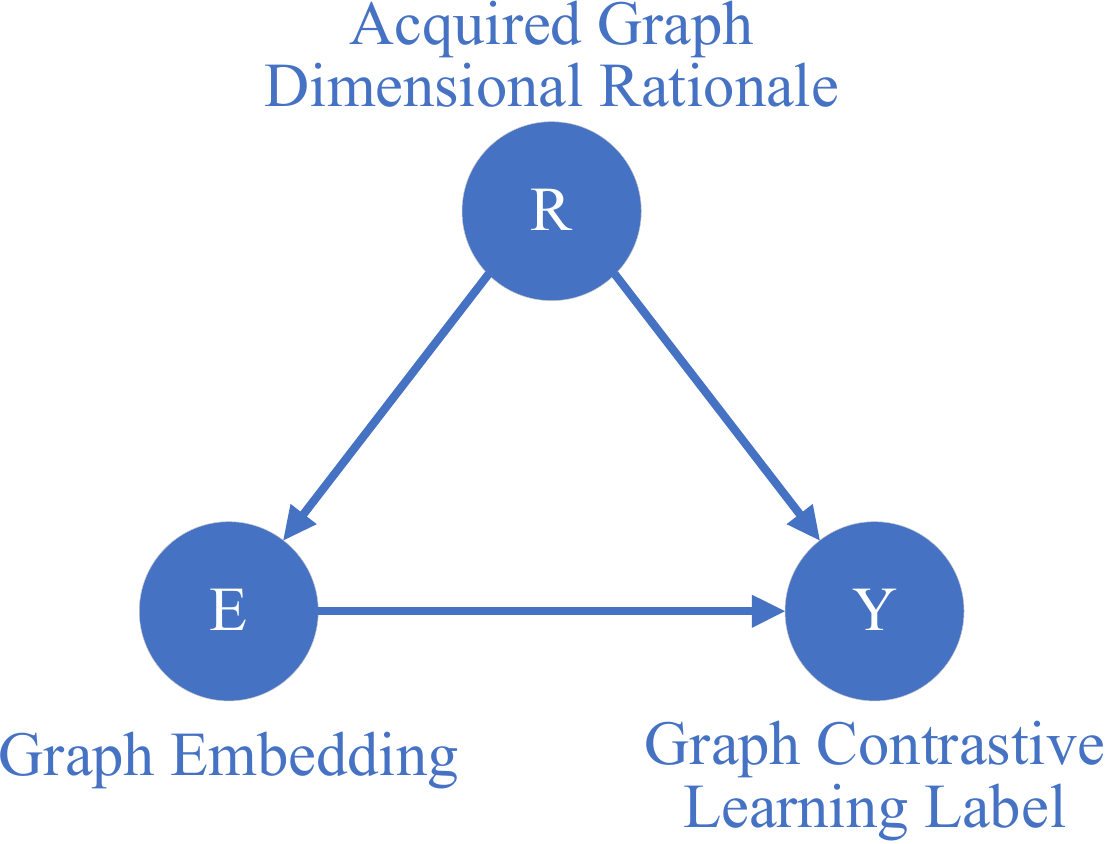}
\caption{SCM for GCL pretraining.}\label{fig:SCM}
\end{center}
\end{figure}

\section{Problem Formulations}
\label{sec:causal part}

\subsection{Graph Contrastive Learning}

Given a graph $\mathcal{G}$ sampled from the dataset of $M$ graphs, denoted as $\mathcal{G} \in \{ \mathcal{G}_m: m \in M \}$, we formulated the augmented graph $\hat{\mathcal{G}}$ by applying the augmentation distribution $\mathcal{T}(\hat{\mathcal{G}} \vert \mathcal{G})$. During pre-training, we sample a minibatch of $N$ graphs from $\mathcal{G}_m$ and denote it as $\mathcal{G}^\prime = \{ \mathcal{G}_n \}_{n=1}^N$, where $\mathcal{G}_{n}$ represents the n-th sample. We perform stochastic data augmentations to transform each sample $\mathcal{G}_{n}$ into two augmented views $\hat{\mathcal{G}}_{n,i}$ and $\hat{\mathcal{G}}_{n,j}$. Then $\hat{\mathcal{G}}_{n,i}$ and $\hat{\mathcal{G}}_{n,j}$ are fed into a feature extractor to get their feature representations $\boldsymbol{z}_{n,i}$ and $\boldsymbol{z}_{n,j}$. Then a GCL loss function is defined to enforce maximizing the consistency between positive pairs $\boldsymbol{z}_{n,i}$, $\boldsymbol{z}_{n,j}$, such as InfoNCE loss \cite{graphcl,simgrace}:
\begin{equation}\label{eq:infonce}\small
\mathcal{L}_{IN}=\sum_{n=1}^N -\log \frac{\exp \left(d\left(\boldsymbol{z}_{n, i}, \boldsymbol{z}_{n, j}\right) / \tau\right)}{\sum_{n^{\prime}=1, n^{\prime} \neq n}^N \exp \left(d\left(\boldsymbol{z}_{n, i}, \boldsymbol{z}_{n^{\prime}, j}\right) / \tau\right)},
\end{equation}
where $\tau$ denotes the temperature parameter and $d\left(\cdot, \cdot\right)$ denotes the cosine similarity function. 

\subsection{Structural Causal Model}
The intuition of our work comes from the investigation of the effects of preserving different DRs in the graph embedding, as shown in Figure \ref{fig:motivation experiment}. According to \cite{chaos}, the cross-entropy loss can be bounded by the contrastive loss, indicating that we can improve the performance on downstream classification tasks by enhancing the discriminability of representations learned by GCL during pre-training. Thus, improving the quality of the acquired DRs during pre-training derives the indirect promotion of the DRs acquired on downstream tasks.
To this end, we establish an SCM to elaborate the causality among the variables in GCL: graph embedding $E$, acquired graph DR $R$, and graph contrastive label $Y$. The SCM is depicted in Figure \ref{fig:SCM}, with each link representing a causality between two variables:
\begin{itemize}
    \item $E \rightarrow Y$. The graph embedding $E$ can directly affect the graph contrastive label $Y$.      
    \item $R \rightarrow E$. The acquired graph DR $R$ affects the learned features of the graph embeddings by contributing to the gradient effect of the training of the graph encoder, which further affects the graph embedding $E$.
    \item $R \rightarrow Y$. In GCL, the graph contrastive learning label is related to the anchor, such that the acquired graph DR $R$ causally affects the anchor due to the proposed iterative training paradigm of DRGCL.
\end{itemize}
According to the SCM, the acquired graph DR $R$ is the causal confounder between $E$ and $Y$ due to the causal effects of $R$ towards $E$ and $Y$.

\begin{figure*}[t]
    \centering	\includegraphics[width=0.8\textwidth]{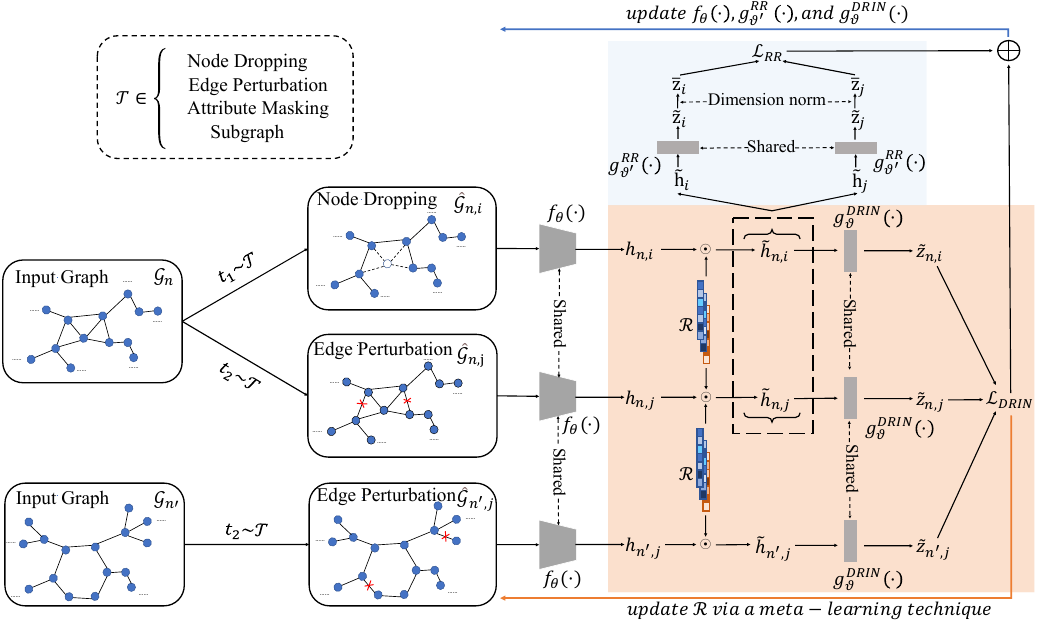}
	\caption{Illustration of DRGCL. 
 The solid blue line pointing backwards represents the regular training step. The solid red line pointing backwards represents the meta-learning step.
 }
	\label{fig:rationale acquiring network}
\end{figure*}

\subsection{Causal Intervention via Backdoor Adjustment}

\cite{pearl2000models} proposes the definition of the backdoor path to demarcate the scope of application of the backdoor criterion. 
In our SCM, there exists a backdoor path $E \leftarrow R \rightarrow Y$, resulting in the spurious correlation between $E$ and $Y$. Then, if we use $P(Y|E)$ to measure the causality between $E$ and $Y$ as the approach adopted by the conventional GCL methods, the task-irrelevant features would affect the downstream classification. To eliminate the causal effect of the backdoor path, we can intervene on the variable $E$ and condition on the confounding factor $R$. The adjustment formula can be written as follows:
\begin{equation}\label{eq:adjustment formula}\small
P(Y|do(E)) =\sum_r P(Y|E, R=r) P(R=r),
\end{equation} 
where $r$ denotes the value of $R$.
\section{Methodology}
\label{sec:method}

In this paper, we focus on developing a novel GCL learning framework which is shown in Figure \ref{fig:rationale acquiring network}. 

\subsection{Graph Dimensional Rationale Acquiring Network}
By following the data augmentations in GraphCL \cite{graphcl}, we sample two transformations $t_1$ and $t_2$ from the augmentation distribution $\mathcal{T}$ and further obtain two correlated views $\hat{\mathcal{G}}_{n,i}$ and $\hat{\mathcal{G}}_{n,j}$. Then we feed them into the GNN-based encoder $f_\theta(\cdot)$ to extract graph-level representations $\boldsymbol{h}_{n,i}$, $\boldsymbol{h}_{n,j}$. To acquire the DRs from the candidate graphs, we introduce a learnable DR weight, denoted as $\mathcal{R} = \left\{ \boldsymbol{\omega}_k \big| k \in \llbracket {1, D} \rrbracket \right\}$, where $D$ represents the number of graph embedding dimensions, which is treated as containing the shared knowledge with the acquired DR. 

\begin{figure}[t]
\begin{center}
\includegraphics[width=0.47\textwidth]{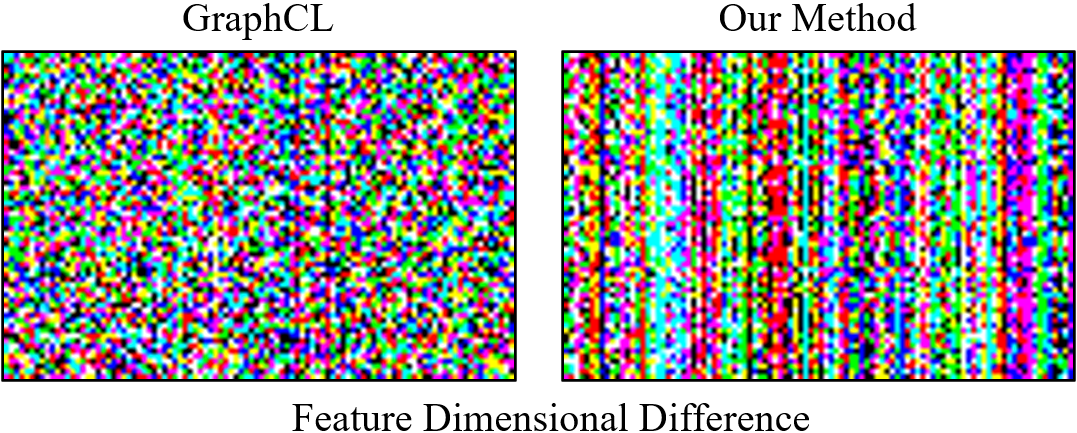}
\caption{The visualization of the representations learned by GraphCL and our method using the redundancy reduction method on the BBBP dataset, respectively. The learned features are projected into a colored image in RGB format, where different colors represent different types of features. The abscissa axis represents the feature dimensions, and the ordinate axis represents samples of different classes. The greater the color contrast, the lower the dimensional feature similarity. These plots represent the similarity between dimension features with the first 64 samples of BBBP.}
\label{fig:dimension-redundancy}
\end{center}
\end{figure}

\begin{equation}\label{eq:hadama}\small
\tilde{\boldsymbol{h}}=\boldsymbol{h} \odot \mathcal{R},
\end{equation}
where $\tilde{\boldsymbol{h}}$ denotes the representation derived by preserving the acquired DR, and $\odot$ represents an element-wise product operation. By utilizing this operation, we obtain $\tilde{\boldsymbol{h}}_{n,i}$ and $\tilde{\boldsymbol{h}}_{n,j}$.
Furthermore, we utilize a projection head $g^{DRIN}_{\vartheta}(\cdot)$ to map the graph representations into a latent space:
\begin{equation}\label{eq:rationale acquire}\small
    \tilde{\boldsymbol{z}}_{n,i} = g^{DRIN}_{\vartheta} (\tilde{\boldsymbol{h}}_{n,i}),  \tilde{\boldsymbol{z}}_{n,j} = g^{DRIN}_{\vartheta}(\tilde{\boldsymbol{h}}_{n,j}).
\end{equation}
Subsequently, we utilize the 
DR-aware InfoNCE loss as 
\begin{equation}\label{eq:drin}\small
\mathcal{L}_{DRIN}=\sum_{n=1}^N -\log \frac{\exp \left(d\left(\tilde{\boldsymbol{z}}_{n, i}, \tilde{\boldsymbol{z}}_{n, j}\right) / \tau\right)}{\sum_{n^{\prime}=1, n^{\prime} \neq n}^N \exp \left(d\left(\tilde{\boldsymbol{z}}_{n, i}, \tilde{\boldsymbol{z}}_{n^{\prime}, j}\right) / \tau\right)}.
\end{equation}

\subsection{Graph Dimensional Rationale Redundancy Reduction}

From the perspective of information theory, each dimension captures a subset of the information entropy of the graph representation. As depicted in Figure \ref{fig:dimension-redundancy}, GraphCL encounters the issue of graph dimensional redundancy, which indicates that multiple dimensions in graph embeddings share overlapping information entropy. To tackle the issue, inspired by classical multivariate analysis methods \cite{hotelling1992relations, ccassg}, we apply the graph DR redundancy reduction to DRGCL. 
Following the aforementioned manner to get $\tilde{\boldsymbol{h}}$, we obtain $\tilde{\mathbf{h}}_i$, $\tilde{\mathbf{h}}_j$, which denote the representations from a minibatch of $N$ graphs of two augmented views $\hat{\mathcal{G}}_i$ and $\hat{\mathcal{G}}_j$. Subsequently, we use a specific projection head $g^{RR}_{\vartheta^\prime}(\cdot)$ to map the graph representations into a latent space:
\begin{equation}\small
    \tilde{\mathbf{z}}_i = g^{RR}_{\vartheta^\prime}(\tilde{\mathbf{h}}_i),  \tilde{\mathbf{z}}_j = g^{RR}_{\vartheta^\prime}(\tilde{\mathbf{h}}_j).
\end{equation}
Then we apply an instance-dimensional normalization to ensure each feature dimension has a $0$-mean and $1/\sqrt{N}$-standard deviation distribution, which is implemented as:
\begin{equation}\small
\bar{\mathbf{z}}=\frac{\tilde{\mathbf{z}}-\mu(\tilde{\mathbf{z}})}{\sigma(\tilde{\mathbf{z}})*\sqrt{N}}.
\end{equation}
The obtained normalized $\bar{\mathbf{z}}_i$ and $\bar{\mathbf{z}}_j$ are further used to form the redundancy reduction loss for a certain graph as
\begin{equation}\small
\mathcal{L}_{RR}=\underbrace{\mathcal{F}( \bar{\mathbf{z}}_i,\bar{\mathbf{z}}_j)}_{\text {invariance term}}+\underbrace{\lambda(\mathcal{F}(\bar{\mathbf{z}}_i^{\top} \bar{\mathbf{z}}_i, \mathbf{I})+\mathcal{F}(\bar{\mathbf{z}}_j^{\top} \bar{\mathbf{z}}_j, \mathbf{I}))}_{\text{decorrelation term}} ,  
\end{equation}
where $\mathcal{F}(\cdot,\cdot)=\left\| \cdot-\cdot \right\|_F^2$, $\left\| \cdot \right\|_F^2$ denotes the Frobenius norm and $\lambda$ is a trade-off hyperparameter. Intuitively, the invariance term makes the embedding invariant to the distortions of a graph by minimizing the difference between two normalized representations. By trying to equate the off-diagonal elements of the auto-correlation matrix of each view's representation to $0$, the decorrelation term reduces the redundancy between the representations, thereby avoiding the collapsed trivial solution outputting the same vector for all inputs.
In Figure \ref{fig:dimension-redundancy}, it can be intuitively observed that adopting the redundancy reduction loss, our DRGCL can indeed learn representations with information-decoupled dimensions.

\begin{algorithm}[t]
\small
 \begin{algorithmic}
  \STATE {\bfseries Input:} Graph dataset $\mathcal{G}_m$ with $M$ graphs, minibatch size $N$, and a hyper-parameter $\alpha$.\\
  \STATE {\bf Initialize} The neural network parameters: $\theta$ for $f_{\theta}(\cdot)$, $\vartheta$ for $g^{DRIN}_{\vartheta}(\cdot)$, $\vartheta^\prime$ for $g^{RR}_{\vartheta^\prime}(\cdot)$, and $\mathcal{R} = \left\{ \boldsymbol{\omega}_k \big| k \in \llbracket {1, D} \rrbracket \right\}$. The learning rates: $\beta_\theta$ and $\beta_\vartheta$, etc.
  \REPEAT
  \FOR{$t$-th training iteration}
  \STATE Iteratively sample a minibatch  $\mathcal{G}^\prime$ with $N$ examples from $\mathcal{G}_m$, $\mathcal{G}^\prime = \{ \mathcal{G}_n: n=1,2,...N \}$
  \STATE Randomly sample two augmentations $t_1$, $t_2$ from $\mathcal{T}$, the augmented views of $\mathcal{G}_n$ can be denoted as $\hat{\mathcal{G}}_{n,i}$ and $\hat{\mathcal{G}}_{n,j}$, the augmented views of $\mathcal{G}^\prime$ can be denoted as $\hat{\mathcal{G}}^{\prime}$, including $\hat{\mathcal{G}}_i$, $\hat{\mathcal{G}}_j$.
  
  \FOR{$n=1$ to $N$}
    \STATE $\tilde{\boldsymbol{h}}_{n,i}=f_\theta(\hat{\mathcal{G}}_{n,i})\odot\mathcal{R}$, \ $\tilde{\boldsymbol{h}}_{n,j}=f_\theta(\hat{\mathcal{G}}_{n,j})\odot\mathcal{R}$
    \STATE $\tilde{\boldsymbol{z}}_{n,i} = g^{DRIN}_{\vartheta}(\tilde{\boldsymbol{h}}_{n,i})$,\ 
    $\tilde{\boldsymbol{z}}_{n,j} = g^{DRIN}_{\vartheta}(\tilde{\boldsymbol{h}}_{n,j})$
  \ENDFOR
  \STATE $\mathcal{L}_{DRIN} = \sum_{n=1}^N -\log \frac{\exp \left(d\left(\tilde{\boldsymbol{z}}_{n, i}, \tilde{\boldsymbol{z}}_{n, j}\right) / \tau\right)}{\sum_{n^{\prime}=1, n^{\prime} \neq n}^N \exp \left(d\left(\tilde{\boldsymbol{z}}_{n, i}, \tilde{\boldsymbol{z}}_{n^{\prime}, j}\right) / \tau\right)}$  
    
  \STATE $\tilde{\mathbf{h}}_i=f_\theta(\hat{\mathcal{G}}_i)\odot\mathcal{R}$,\ 
  $\tilde{\mathbf{h}}_j=f_\theta(\hat{\mathcal{G}}_j)\odot\mathcal{R}$
  \STATE $\tilde{\mathbf{z}}_i = g^{RR}_{\vartheta^\prime}(\tilde{\mathbf{h}}_i)$,\ $\tilde{\mathbf{z}}_j = g^{RR}_{\vartheta^\prime}(\tilde{\mathbf{h}}_j)$
  \STATE $\bar{\mathbf{z}}_i=\frac{\tilde{\mathbf{z}}_i-\mu(\tilde{\mathbf{z}}_i)}{\sigma(\tilde{\mathbf{z}}_i)*\sqrt{N}}$,
  $\bar{\mathbf{z}}_j=\frac{\tilde{\mathbf{z}}_j-\mu(\tilde{\mathbf{z}}_j)}{\sigma(\tilde{\mathbf{z}}_j)*\sqrt{N}}$
  \STATE 
  $\mathcal{L}_{RR}=\underbrace{\mathcal{F}( \bar{\mathbf{z}}_i,\bar{\mathbf{z}}_j)}_{\text {invariance term}}+\underbrace{\lambda(\mathcal{F}(\bar{\mathbf{z}}_i^{\top} \bar{\mathbf{z}}_i, \mathbf{I})+\mathcal{F}(\bar{\mathbf{z}}_j^{\top} \bar{\mathbf{z}}_j, \mathbf{I}))}_{\text{decorrelation term}}$

  \STATE $\# \ regular \ training \ step, \ fix \ \mathcal{R}$ 
  
  \STATE $\mathop{\arg\min}_{\theta, \vartheta, \vartheta^\prime} \mathcal{L}_{RR} + \alpha \cdot \mathcal{L}_{DRIN}$
  \STATE $\# \ compute \ trial \ weights \ and \ retain \ computational$
  \STATE $\# \ graph, fix \ \theta \ and \ \vartheta$
  \STATE $\theta_{trial}=\theta-\beta_\theta \nabla_\theta \mathcal{L}_{DRIN}\left(g_{\vartheta}^{DRIN }\left(f_\theta\left(\hat{\mathcal{G}}^{\prime}\right) \odot \mathcal{R}\right)\right)$, \\
$\vartheta_{trial}=\vartheta-\beta_{\vartheta} \nabla_{\vartheta} \mathcal{L}_{DRIN}\left(g_{\vartheta}^{DRIN}\left(f_\theta\left(\hat{\mathcal{G}}^{\prime}\right) \odot \mathcal{R}\right)\right)$
  \STATE $\# \ meta \ training \ step \ using \ second \ derivative$
  \STATE $\underset{\mathcal{R}}{\arg \min} \mathcal{L}_{DRIN }\left(g^{DRIN}_{\vartheta_{trial}}\left(f_{\theta_{trial }}\left(\hat{\mathcal{G}}^{\prime}\right) \odot \mathcal{R}\right)\right)$ \\
  \ENDFOR
  \UNTIL $\theta$, $\vartheta$, $\vartheta^{RR}$, and $\mathcal{R}$ converge.
 \end{algorithmic}
 \caption{The DRGRL training algorithm}
 \label{alg:DRGRL}
\end{algorithm}

\begin{figure}[t]
\begin{center}
\includegraphics[width=0.32\textwidth]{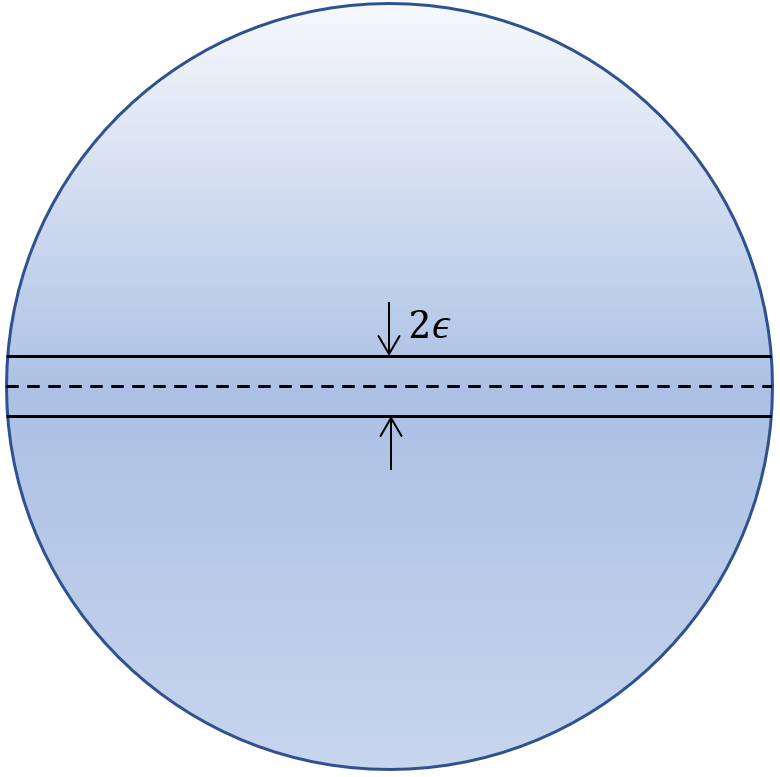}
\caption{A counter-intuitive high-dimensional phenomenon in the problem of measuring concentration on a sphere. Almost the whole area of a high-dimensional sphere is concentrated in an $\epsilon$-strip around its equator and actually around any great circle.}
\label{fig:sphere}
\end{center}
\end{figure}

\begin{table*}[t]
	\begin{center}
		\begin{small}
			\begin{tabular}{c|cccc|cccc|c}
				\hline
                    \hline
				\text{Dataset} & \text{NCI1} & \text{PROTEINS} & \text{DD} & \text{MUTAG} & \text{COLLAB} & \text{RDT-B}
                        & \text{RDT-M5K} & \text{IMDB-B} & A.R. $\downarrow$\\
				\hline
                    \hline
				\text{GL}  & - & - & - & 81.7 $\pm$ 2.1 & -
                        & 77.3 $\pm$ 0.2 & 41.0 $\pm$ 0.2 & 69.9 $\pm$ 1.0 & 11.0\\
                    \text{WL}  & \textbf{80.0} $\pm$ 0.5 & 73.0 $\pm$ 0.6 & - & 80.7 $\pm$ 3.0 & -
                        & 68.8 $\pm$ 0.4 & 46.1 $\pm$ 0.2 & \textbf{72.3} $\pm$ 3.4 & 8.3\\
				\text{DGK}  & \textbf{80.3} $\pm$ 0.5 & 73.3 $\pm$ 0.8 & - & 87.4 $\pm$ 0.7 & -
                        & 78.0 $\pm$ 0.4 & 41.3 $\pm$ 0.2 & 67.0 $\pm$ 0.6 & 8.0\\
                    \hline
                    \text{node2vec}  & 54.9 $\pm$ 1.6 & 57.5 $\pm$ 3.6 & - & 72.6 $\pm$ 10.0 & -
                        & - & - & - & 12.3\\
                    \text{sub2vec}  & 52.8 $\pm$ 1.5 & 53.0 $\pm$ 5.6 & - & 61.1 $\pm$ 15.8 & -
                        & 71.5 $\pm$ 0.4 & 36.7 $\pm$ 0.4 & 55.3 $\pm$ 1.5 & 13.0\\
                    \text{graph2vec}  & 73.2 $\pm$ 1.8 & 73.3 $\pm$ 2.0 & - & 83.2 $\pm$ 9.3 & -
                        & 75.8 $\pm$ 1.0 & 47.9 $\pm$ 0.3 & 71.1 $\pm$ 0.5 & 9.3\\
                    \text{InfoGraph}  & 76.2 $\pm$ 1.0 & 74.4 $\pm$ 0.3 & 72.9 $\pm$ 1.8 & \textbf{89.0} $\pm$ 1.1 & 70.7 $\pm$ 1.1
                        & 82.5 $\pm$ 1.4 & 53.5 $\pm$ 1.0 & \textbf{73.0} $\pm$ 0.9 & 5.8\\
                    \text{GraphCL}  & 77.9 $\pm$ 0.4 & 74.4 $\pm$ 0.5 & \textbf{78.6} $\pm$ 0.4 & 86.8 $\pm$ 1.3 & \textbf{71.4} $\pm$ 1.2
                        & 89.5 $\pm$ 0.8 & \textbf{56.0} $\pm$ 0.3 & 71.2 $\pm$ 0.4 & 5.0\\
                    \text{ADGCL}  & 73.9 $\pm$ 0.8 & 73.3 $\pm$ 0.5 & 75.8 $\pm$ 0.9 & 88.7 $\pm$ 1.9 & \textbf{72.0} $\pm$ 0.6
                        & \textbf{90.1} $\pm$ 0.9 & 54.3 $\pm$ 0.3 & 70.2 $\pm$ 0.7 & 6.1\\
                    \text{JOAO}  & 78.1 $\pm$ 0.5 & 74.6 $\pm$ 0.4 & 77.3 $\pm$ 0.5 & 87.4 $\pm$ 1.0 & 69.5 $\pm$ 0.4
                        & 85.3 $\pm$ 1.4 & 55.7 $\pm$ 0.6 & 70.2 $\pm$ 3.1 & 6.5\\
                    \text{JOAOv2}  & 78.4 $\pm$ 0.5 & 74.1 $\pm$ 1.1 & 77.4 $\pm$ 1.2 & 87.7 $\pm$ 0.8 & 69.3 $\pm$ 0.3
                        & 86.4 $\pm$ 1.5 & \textbf{56.0} $\pm$ 0.3 & 70.8 $\pm$ 0.3 & 5.8\\
                    \text{RGCL}  & 78.1 $\pm$ 1.0 & \textbf{75.0} $\pm$ 0.4 & \textbf{78.9} $\pm$ 0.5 & 87.7 $\pm$ 1.0 & 71.0 $\pm$ 0.7
                        & \textbf{90.3} $\pm$ 0.6 & \textbf{56.4} $\pm$ 0.4 & 71.9 $\pm$ 0.9 & \textbf{3.3}\\
                    \text{SimGRACE}  & \textbf{79.1} $\pm$ 0.4 & \textbf{75.3} $\pm$ 0.1 & 77.4 $\pm$ 1.1 & \textbf{89.0} $\pm$ 1.3 & \textbf{71.7} $\pm$ 0.8
                        & 89.5 $\pm$ 0.9 & 55.9 $\pm$ 0.3 & 71.3 $\pm$ 0.8 & \textbf{3.3}\\
				\hline
                    \textbf{\text{DRGCL}} & 78.7 $\pm$ 0.4 & \textbf{75.2} $\pm$ 0.6 & \textbf{78.4} $\pm$ 0.7 & \textbf{89.5} $\pm$ 0.6 & 70.6 $\pm$ 0.8 & \textbf{90.8} $\pm$ 0.3 & \textbf{56.3} $\pm$ 0.2 & \textbf{72.0} $\pm$ 0.5 & \textbf{2.8} \\
                    \hline
                    \hline
			\end{tabular}
		\end{small}
	\end{center}
 \caption{Unsupervised representation learning classification accuracy (\%) on TU datasets (mean 10-fold cross-validation accuracy with 5 runs). A.R denotes the average rank of the results. The top-3 results are highlighted in \textbf{bold}.}
	\label{tab:unsupervised learning}
\end{table*}

\begin{table*}[t]
	\begin{center}
		\begin{small}
			\begin{tabular}{c|cccccccc|c}
				\hline
                    \hline
				\text{Dataset} & \text{BBBP} & \text{Tox21} & \text{ToxCast} & \text{SIDER} & \text{ClinTox} & \text{MUV}
                        & \text{HIV} & \text{BACE} & AVG.\\
				\hline
                    \hline
				\text{No Pre-Train}  & 65.8 $\pm$ 4.5 & 74.0 $\pm$ 0.8 & 63.4 $\pm$ 0.6 & 57.3 $\pm$ 1.6 & 58.0 $\pm$ 4.4
                        & 71.8 $\pm$ 2.5 & 75.3 $\pm$ 1.9 & 70.1 $\pm$ 5.4 & 67.0\\
				\hline
                    \text{AttrMasking}  & 64.3 $\pm$ 2.8 & \textbf{76.7} $\pm$ 0.4 & \textbf{64.2} $\pm$ 0.5 & \textbf{61.0} $\pm$ 0.7 & 71.8 $\pm$ 4.1
                        & \textbf{74.7} $\pm$ 1.4 & 77.2 $\pm$ 1.1 & \textbf{79.3} $\pm$ 1.6 & 71.1\\
                    \text{ContextPred}  & 68.0 $\pm$ 2.0 & \textbf{75.7} $\pm$ 0.7 & \textbf{63.9} $\pm$ 0.6 & 60.9 $\pm$ 0.6 & 65.9 $\pm$ 3.8
                        & \textbf{75.8} $\pm$ 1.7 & 77.3 $\pm$ 1.0 & \textbf{79.6} $\pm$ 1.2 & 70.9\\
                    \text{GraphCL}  & 69.7 $\pm$ 0.7 & 73.9 $\pm$ 0.7 & 62.4 $\pm$ 0.6 & 60.5 $\pm$ 0.9 & 76.0 $\pm$ 2.7
                        & 69.8 $\pm$ 2.7 & \textbf{78.5} $\pm$ 1.2 & 75.4 $\pm$ 1.4 & 70.8\\
                    \text{ADGCL}  & 68.3 $\pm$ 1.0 & 73.6 $\pm$ 0.8 & 63.1 $\pm$ 0.7 & 59.2 $\pm$ 0.9 & 77.6 $\pm$ 4.2
                        & \textbf{74.9} $\pm$ 2.5 & 75.5 $\pm$ 1.3 & 75.0 $\pm$ 1.9 & 70.9 \\
                    \text{JOAO} & 70.2 $\pm$ 1.0 & 75.0 $\pm$ 0.3 & 63.0 $\pm$ 0.5 & 60.0 $\pm$ 0.8 & \textbf{81.3} $\pm$ 2.5
                        & 71.7 $\pm$ 1.4 & 76.7 $\pm$ 1.2 & 77.3 $\pm$ 0.5 & 71.9 \\
                    \text{JOAOv2}  & \textbf{71.4} $\pm$ 0.9 & 74.2 $\pm$ 0.6 & 63.2 $\pm$ 0.5 & 60.5 $\pm$ 0.7 & \textbf{81.0} $\pm$ 1.6
                        & 73.7 $\pm$ 1.0 & 77.5 $\pm$ 1.2 & 75.5 $\pm$ 1.3 & \textbf{72.1}\\
                    RGCL$^{\ddag}$ & \textbf{71.4} $\pm$ 0.7 & \textbf{75.2} $\pm$ 0.3 & 63.3 $\pm$ 0.2 & \textbf{61.4} $\pm$ 0.6 & 76.4 $\pm$ 3.4
                        & 72.6 $\pm$ 1.5 & \textbf{77.9} $\pm$ 0.8 & 76.0 $\pm$ 0.8 & \textbf{71.8}\\
                    SimGRACE$^{\ddag}$  & \textbf{71.3} $\pm$ 0.9 & 73.9 $\pm$ 0.4 & 63.4 $\pm$ 0.5 & 60.6 $\pm$ 1.0 & 64.0 $\pm$ 1.2
                        & 69.4 $\pm$ 1.2 & 75.0 $\pm$ 1.1 & 74.6 $\pm$ 0.7 & 69.0 \\
				\hline
                    \textbf{\text{DRGCL}}  & \textbf{71.2} $\pm$ 0.5 & 74.7 $\pm$ 0.5 & \textbf{64.0} $\pm$ 0.5 & \textbf{61.1} $\pm$ 0.8 & \textbf{78.2 $\pm$ 1.5} & 73.8 $\pm$ 1.1 & \textbf{78.6} $\pm$ 1.0 & \textbf{78.2} $\pm$ 1.0 & \textbf{72.5}\\
                    \hline
                    \hline
			\end{tabular}
		\end{small}
	\end{center}
        \caption{Transfer learning performance on molecular property prediction in ZINC-2M (mean ROC-AUC + std over 10 runs). AVG. denotes
    the average result in all datasets.
    $\ddag$ means there exist differences in producing the results. RGCL  finetunes ClinTox for 300 epochs and MUV for 50 epochs. For fairness, we reproduce them by finetuning for 100 epochs. SimGRACE only provides the results for BBBP, ToxCast, and SIDER. We provide the results of SimGRACE on other datasets in benchmarks.
        }
	\label{tab:transfer learning}
\end{table*}

\subsection{Dimensional Rationale-aware Graph Contrastive Learning with Backdoor Adjustment} \label{sec:drgclba}

During pre-training, a conventional training paradigm and a meta-learning training paradigm are iteratively employed. Specifically, we train the encoder $f_\theta(\cdot)$, and the projection heads $g^{DRIN}_{\vartheta}(\cdot)$ and $g^{RR}_{\vartheta^\prime}(\cdot)$ in a conventional manner, while the DR weight $\mathcal{R}$ is trained by adopting the meta-learning objective. The overall training procedure of DRGCL consists of two steps.
In the first training step, we follow the standard contrastive learning approach to train $f_\theta(\cdot)$, $g^{DRIN}_{\vartheta}(\cdot)$, and $g^{RR}_{\vartheta^\prime}(\cdot)$. This involves jointly minimizing the contrastive loss and the redundancy reduction loss:
\begin{equation}\small
    \mathcal{L}_{DRGCL} = \mathcal{L}_{RR} + \alpha \cdot \mathcal{L}_{DRIN},
\end{equation}
where $\alpha$ is a hyper-parameter that governs the trade-off between the two loss components. 
The second training step is based on meta-learning. We use a second-derivative technique \cite{sslmeta} to solve a bi-level optimization problem. 
We encourage $\mathcal{R}$ to re-weight the specific dimensions to preserve task-relevant information, which is regarded as the DR for graph embeddings, so that DRGCL can perform the causal intervention via backdoor adjustment during training. 
Specifically, $\mathcal{R}$ is updated by computing its gradients with respect to the performance of $f_\theta(\cdot)$ and $g^{DRIN}_\vartheta(\cdot)$. The corresponding performance is measured by using the gradients of $f_\theta(\cdot)$ and $g^{DRIN}_\vartheta(\cdot)$ during the back-propagation of graph contrastive loss. Based on this updating mechanism during pre-training, the iterations of $\mathcal{R}$ can include sufficient values to perform the backdoor adjustment conditional on $R$ with respect to $E$ and $Y$. Formally, we update the DR weight $\mathcal{R}$ by
\begin{equation} \label{eq:all loss}\small
\underset{\mathcal{R}}{\arg \min} \mathcal{L}_{DRIN }\left(g^{DRIN}_{\vartheta_{trial}}\left(f_{\theta_{trial }}\left(\hat{\mathcal{G}}^{\prime}\right) \odot \mathcal{R}\right)\right),
\end{equation}
where $\hat{\mathcal{G}}^{\prime}$, including $\hat{\mathcal{G}}_i$, $\hat{\mathcal{G}}_j$, denotes the augmented views of $\mathcal{G}^\prime$, and $\mathcal{G}^\prime$ is sampled from the graph dataset $\mathcal{G}_m$. $\theta_{trial}$ and $\vartheta_{trial}$ denote the $trial$ weights of the encoders and projection heads, respectively, after one gradient update using the contrastive loss defined in Equation \ref{eq:drin}. The update of these trial weights is formulated as follows:
\begin{equation}\small
\begin{aligned}
\theta_{trial}&=\theta-\beta_\theta \nabla_\theta \mathcal{L}_{DRIN}\left(g_{\vartheta}^{DRIN }\left(f_\theta\left(\hat{\mathcal{G}}^{\prime}\right) \odot \mathcal{R}\right)\right), \\
\vartheta_{trial}&=\vartheta-\beta_{\vartheta} \nabla_{\vartheta} \mathcal{L}_{DRIN}\left(g_{\vartheta}^{DRIN}\left(f_\theta\left(\hat{\mathcal{G}}^{\prime}\right) \odot \mathcal{R}\right)\right),
\end{aligned}
\end{equation}
where $\beta_\theta$ and $\beta_\vartheta$ are learning rates. The intuition behind such a behavior is to leverage the second-derivative trick, which involves computing a derivative over the derivative of the combination ${ \theta,\vartheta }$ in order to update $\mathcal{R}$. Specifically, we compute the derivative with respect to $\mathcal{R}$ using a retained computational graph of ${ \theta,\vartheta }$ and then update the DR weight $\mathcal{R}$ by back-propagating this derivative as defined in Equation \ref{eq:all loss}.
Intuitively, the initialization of $\mathcal{R}$ is biased. During pre-training, $\mathcal{R}$ is updated per batch over epochs, resulting in the acquirement of local $\mathcal{R}$ with sufficient self-supervision for the current batch. After pre-training, all graphs have gradient contributions to $\mathcal{R}$, thereby achieving the global DR. 
The two steps for updating $f_\theta(\cdot)$, $g^{DRIN}_\vartheta(\cdot)$, $g^{RR}_{\vartheta^\prime}(\cdot)$ and updating $\mathcal{R}$ are iteratively imposed until convergence. The Algorithm of the training pipeline is detailed in Algorithm  \ref{alg:DRGRL}.

For the fitting on downstream tasks, we utilize the graph DR-aware embeddings for downstream tasks. 

\begin{table*}[t]
    \begin{center}
        \begin{small}
            \begin{tabular}{c|ccccccc|c}
            \hline\hline
                Dataset & No PreTrain & AttrMasking & ContextPred & GraphCL & JOAO & JOAOv2 & SimGRACE & \textbf{DRGCL} \\
                \hline
                PPI-306K & 64.8 $\pm$ 2.0 & 65.2 $\pm$ 1.6 & 64.4 $\pm$ 1.3 & \textbf{67.9} $\pm$ 0.9 & 64.4 $\pm$ 1.4 & 63.9 $\pm$ 1.6 & \textbf{70.3} $\pm$ 1.2 & \textbf{69.4} $\pm$ 0.4 \\
            \hline\hline
            \end{tabular}
        \end{small}
    \end{center}
    \caption{Transfer leaning performance on protein function prediction in biology PPI-306K dataset. The top-3 results are highlighted in \textbf{bold}.}
    \label{tab:transfer learning on bio dataset}
\end{table*}

\begin{table*}[t]
	\begin{center}
		\begin{small}
			\begin{tabular}{c|cccccccc|c}
				\hline
                    \hline
				\text{Dataset} & \text{BBBP} & \text{Tox21} & \text{ToxCast} & \text{SIDER} & \text{ClinTox} & \text{MUV}
                        & \text{HIV} & \text{BACE} & AVG.\\
				\hline
                \hline
			     \text{w/o RR \& DR} & 69.7 $\pm$ 0.7 & 73.9 $\pm$ 0.7 & 62.4 $\pm$ 0.6 & 60.5 $\pm$ 0.9 & 76.0 $\pm$ 2.7
                        & 69.8 $\pm$ 2.7 & 78.5 $\pm$ 1.2 & 75.4 $\pm$ 1.4 & 70.8\\      
                    \text{w/o RR}  & 69.7 $\pm$ 0.7 & 74.7 $\pm$ 0.4 & 63.6 $\pm$ 0.5 & 59.9 $\pm$ 0.4  & 75.6 $\pm$ 3.5
                        & 72.2 $\pm$ 1.6 & 76.9 $\pm$ 0.8 & 75.1 $\pm$ 0.8 & 71.0 \\
                    \text{w/o DR}  & 70.6 $\pm$ 0.8 & 74.3 $\pm$
                     0.5 & 63.8 $\pm$ 0.5 & 60.3 $\pm$ 0.5 & 77.4 $\pm$ 1.4
                        & 73.8 $\pm$ 1.1 & 78.3 $\pm$ 1.0 & 76.8 $\pm$ 0.9 & 71.9\\
                    \hline
                    \textbf{\text{DRGCL}}  & 71.2 $\pm$ 0.5 & 74.7 $\pm$ 0.5 & 64.0 $\pm$ 0.5 & 61.1 $\pm$ 0.8 & 78.2 $\pm$ 1.5 & 73.8 $\pm$ 1.1 & 78.6 $\pm$ 1.0 & 78.2 $\pm$ 1.0 & 72.5\\
                    \hline
                    \hline
			\end{tabular}
		\end{small}
	\end{center}
        \caption{Ablation study for DRGCL on downstream transfer learning.}
	\label{tab:transfer learning ablation}
\end{table*}

\section{Theoretical Analyses}

\subsection{Discussion on Relation between SR and DR}
To facilitate comprehension, we recap the necessary preliminaries of GNN as follows. Suppose that $G=(\mathcal{V}, \mathcal{E})$ is a graph instance with the edge set $\mathcal{E}$ and the node set $\mathcal{V}$. The unified GNN framework follows a neighborhood aggregation strategy, where the representation of a node is iteratively updated by aggregating representations of its neighbors \cite{gin}.
After undergoing k iterations of aggregation, the representation of a node effectively captures the structural information present within its k-hop network neighborhood.
Formally, the $k$-th layer of a GNN is 
\begin{equation}\label{eq:node-repre}
\small
\begin{aligned}
\boldsymbol{a}_v^{(k)} &=\operatorname{AGGREGATE}^{(k)}\left(\left\{\boldsymbol{h}_u^{(k-1)}: u \in \mathcal{N}(v)\right\}\right), \\
\quad \boldsymbol{h}_v^{(k)}&=\operatorname{COMBINE}^{(k)}\left(\boldsymbol{h}_v^{(k-1)}, \boldsymbol{a}_v^{(k)}\right),
\end{aligned}
\end{equation}
where 
$\mathcal{N}(v)$ is a set of nodes adjacent to $v$,
$\boldsymbol{a}_v^{(k)}$ is an aggregating representations of $v$'s neighbors,
$\boldsymbol{h}_v^{(k)}$ 
is the feature vector of node $v$ at the $k$-th layer. 
For graph classification, the READOUT function aggregates node features from the final iteration to obtain the entire graph’s representation $\boldsymbol{h}$:
\begin{equation}\label{eq:readout}
\small
\boldsymbol{h}=\operatorname{READOUT}\left(\left\{\boldsymbol{h}_v^{(k)} \mid v \in G\right\}\right),
\end{equation}
where READOUT can be a simple permutation invariant function such as summation or a more sophisticated
graph-level pooling function.

Our DR applies a dimensional weight to the graph representation while the SR concentrates the rationale in message passing or node representation.
Suppose the AGGREGATE, COMBINE, and READOUT functions are injective, then obviously the change of nodes is a degeneration or special solution of the graph. For ease of discussion, performing attribute masking on node embeddings is equivalent to setting the weight of corresponding dimensions to zero in the graph embeddings.
In addition, as the dimensionality of the representation increases, the representational space of the DR method is expansible. In contrast, SR, being a degenerate form of DR, exhibits a fixed representation space owing to its dependence on the representation space of the underlying graph.
Thereby, DR methods can contain more information entropy, which helps the model to acquire more fine-grained and intrinsic rationales of graphs.

\subsection{Theoretical Feasibility of the Innate Mechanism of the DR}
According to \cite{wright2022high}, high-dimensional problems can be solved with low dimensions. To understand this, we can obtain inspiration from Figure \ref{fig:sphere}, which is a counter-intuitive high-dimensional phenomenon in the problem of measuring concentration on a sphere \cite{DBLP:books/daglib/0018467}. 
Figure \ref{fig:sphere} depicts an $\epsilon$-strip surrounding the equatorial great circle of the sphere $\mathbb{S}^{n-1}$ in $\mathbb{R}^n$. In this case, the great circle corresponds to the equator, where $x_n = 0$. To ensure that the strip covers a significant portion, let's say 99\% of the sphere's area, we have:
\begin{equation}\small
    Area \{ \boldsymbol{x} \in \mathbb{S}^{n-1}: -\epsilon \leq x_n \leq \epsilon \} = 0.99 \cdot Area \left( \mathbb{S}^{n-1} \right).
\end{equation}
Empirical evidence from low-dimensional spheres suggests that a large value of $\epsilon$ is necessary. However, a straightforward calculation reveals that as the dimension $n$ increases, $\epsilon$ decreases on the order of $n^{-1/2}$. Consequently, as $n$ becomes large, the width of the strip $2\epsilon$ can become arbitrarily small. Consequently, as illustrated in Figure \ref{fig:sphere}, the majority of the sphere's area concentrates around the equator.

By the same token, obtaining discriminative information from high-dimensional graph embeddings can be solved with low dimensions. The process of obtaining the DR can be regarded as detecting the point distribution of the sphere. In extreme cases, only a few dimensions of graph embeddings contribute to the downstream task, i.e., many dimensional weights are approaching 0. Then, our graph representation problem can be equivalent to the problem of measuring concentration on a sphere in Figure \ref{fig:sphere}.

\begin{table*}[t]
	\begin{center}
		\begin{small}
			\begin{tabular}{c|cccccc|c}
				\hline
                    \hline
				\text{Fixed $\mathcal{R}$} & \text{BBBP} & \text{Tox21} & \text{ToxCast} & \text{SIDER} & \text{ClinTox} & \text{BACE} & AVG.\\
				\hline
                \hline
			     \text{0.1} & 69.2 $\pm$ 1.0 & 75.3 $\pm$ 0.3 & 63.3 $\pm$ 0.3 & 60.6 $\pm$ 0.7 & 79.1 $\pm$ 1.4
                        & 75.3 $\pm$ 1.2 & 70.5 \\      
                    \text{0.3}  & 69.7 $\pm$ 0.9 & 74.1 $\pm$ 0.4 & 63.7 $\pm$ 0.4 & 61.6 $\pm$ 0.7 & 80.8 $\pm$ 1.2
                        & 76.3 $\pm$ 0.9 & 71.0 \\
                    \text{0.7}  & 69.7 $\pm$ 0.5 & 75.0 $\pm$
                     0.4 & 63.7 $\pm$ 0.5 & 60.0 $\pm$ 0.3 & 79.0 $\pm$ 1.9
                        & 73.7 $\pm$ 1.1 & 70.2\\
                    \text{1.0} & 70.6 $\pm$ 0.8 & 74.3 $\pm$
                     0.5 & 63.8 $\pm$ 0.5 & 60.3 $\pm$ 0.5 & 77.4 $\pm$ 1.4
                        & 76.8 $\pm$ 0.9 & 70.5\\
                    \hline
                    \textbf{\text{DRGCL}}  & 71.2 $\pm$ 0.5 & 74.7 $\pm$ 0.5 & 64.0 $\pm$ 0.5 & 61.1 $\pm$ 0.8 & 78.2 $\pm$ 1.5 & 
                    78.2 $\pm$ 1.0 & 71.2 \\
                    \hline
                    \hline
			\end{tabular}
		\end{small}
	\end{center}
        \caption{Transfer learning in ZNIC-2M with different fixed $\mathcal{R}$.}
	\label{tab:transfer learning with different R}
\end{table*}

\subsection{Guarantees for DRGCL's Effectiveness}

Motivated by \cite{chaos,metamask}, we provide two Theorems as guarantees for DRGCL's effectiveness in the field of self-supervised graph representation learning.
Theorem \ref{the:connect} states that reducing the risk of GCL loss can improve the performance on downstream tasks, supporting our intuition to make the model focus on the acquisition of discriminative information by learning a DR weight.
Theorem \ref{the:tighter} states that given the label $\boldsymbol{y}$, the DR-aware representation $\tilde{\boldsymbol{z}}$ has smaller conditional variance than $\boldsymbol{z}$ in conventional GCL. 
Two Theorems are formulated as follows:

\begin{theorem}
	(Connecting Graph DR-aware Representations to Downstream Cross-Entropy Loss). Under the minimal assumption of GCL, i.e., the graph contrastive label is invariant to the distributions, when $\mathcal{R}$ is optimal, for any $\tilde{\boldsymbol{z}} \in \mathbb{R}$, the cross-entropy loss $\mathcal{L}_{CE}^\mu \left(\tilde{\boldsymbol{z}}\right)$ for downstream classification can be bounded by $\mathcal{L}_{DRIN}\left(\tilde{\boldsymbol{z}}\right)$:
    \small	
    \begin{equation}
		\begin{aligned}			&\mathcal{L}_{DRIN}\left(\tilde{\boldsymbol{z}}\right) - \sqrt{\psi\left(\tilde{\boldsymbol{z}} \Big| \boldsymbol{y}\right)} - \frac{1}{2} \psi\left(\tilde{\boldsymbol{z}} \Big| \boldsymbol{y}\right) - Err \\ &\leq \mathcal{L}_{CE}^\mu \left(\tilde{\boldsymbol{z}}\right) + \log \left(M/D\right) 
            \leq 	\mathcal{L}_{DRIN}\left(\tilde{\boldsymbol{z}} \right) + \sqrt{\psi\left(\tilde{\boldsymbol{z}} \Big| \boldsymbol{y}\right)} + Err,
		\end{aligned}
    \end{equation}
where $M$ is negative samples' quantity, 
$D$ denotes the representation's dimensionality, 
$\tilde{\boldsymbol{z}}$ is the DR-aware representation, $\boldsymbol{y}$ is the \textit{target} label, $\small \psi \left( \tilde{\boldsymbol{z}} \Big| \boldsymbol{y}  \right) $ is the conditional feature variance, and $\small Err = \mathcal{O} \left( M^{-1/2} \right)$ is the approximation error's order.
	\label{the:connect}
\end{theorem}
\begin{theorem}
	(Guarantees for Reduced Conditional Variance of Graph DR-aware Representations). When $\mathcal{R}$ is optimal, for any coupled $\boldsymbol{z}, \tilde{\boldsymbol{z}} \in \mathbb{R}$, given label $\boldsymbol{y}$, the conditional variance of $\tilde{\boldsymbol{z}}$ is reduced:
	\label{the:tighter}
	\begin{equation}\label{eq:tight var}\small   
		\psi\left(\tilde{\boldsymbol{z}} \Big| \boldsymbol{y}\right) \leq \psi\left(\boldsymbol{z} \Big| \boldsymbol{y}\right), \ yet \ \ 
        \psi\left(\left(\tilde{\boldsymbol{z}}\right)^k \Big| \boldsymbol{y}\right) \cong \psi\left(\left(\boldsymbol{z}\right)^k \Big| \boldsymbol{y}\right),
	\end{equation}
    where $\left( \cdot \right)^k$ is a function acquiring $k$-th dimension vector.
\end{theorem}

\begin{proof}
Suppose our redundancy reduction part can best decorrelate the dimensions in graph embeddings, we have
\begin{align}\small
\psi \left( \tilde{\boldsymbol{z}} \Big| y \right) 
&\overset{(1)}{=} \sum_{k=1}^D \psi \left( \left( \tilde{\boldsymbol{z}} \right)^k \Big| y \right) \\ 
&\overset{(2)}{=} \sum_{k=1}^D \psi \left( \omega_k \left( \boldsymbol{z} \right)^k \Big| y \right) \\
&\overset{(3)}{=} \sum_{k=1}^D \omega_k^2 \psi \left( \left( \boldsymbol{z} \right)^k \Big| y \right) \\
&\overset{(4)}{\leq} \sum_{k=1}^D \psi \left( \left( \boldsymbol{z} \right)^k \Big| y \right)
\overset{(5)}{=} \psi \left( \boldsymbol{z} \Big| y \right),
\end{align}
where (1) holds because each dimension is independent of the others;
(2) is derived by Equation \ref{eq:hadama} and Equation \ref{eq:rationale acquire};
(3) is acquired by the property of variance that $\psi(Ax)=A^2\psi(x)$ if $A$ is a random variable;
(4) holds because $\omega_k \leq 1$;
(5) holds due to Equation 14 in \cite{chaos}. 
\end{proof}
\cite{metamask} has already proved the equality part of Equation \ref{eq:tight var}. 
Thus, we further provide the proof for the inequality part in Equation \ref{eq:tight var} as above.
We incorporate Theorem \ref{the:tighter} into Theorem \ref{the:connect} in order to infer an outcome: our methodology can more effectively limit the downstream classification risk. This means the upper and lower limits of supervised cross-entropy loss established by DRGCL are more constrained compared to those obtained through conventional GCL techniques.

\section{Experiments}\label{sec:experiments}

\subsection{Experimental Setup}
For unsupervised learning, we benchmark DRGCL on eight established datasets in TU datasets \cite{tudataset}. 
The baselines include Graphlet Kernel (GL) \cite{gl}, Weisfeiler-Lehman Sub-tree Kernel (WL) \cite{wl}, Deep Graph Kernels (DGK) \cite{dgk}, Node2Vec \cite{node2vec}, Sub2Vec \cite{sub2vec}, Graph2Vec \cite{graph2vec}, InfoGraph \cite{infograph}, GraphCL \cite{graphcl}, ADGCL \cite{adgcl}, JOAO \cite{joao}, RGCL \cite{rgcl} and SimGRACE \cite{simgrace}.
For transfer learning, we perform pre-training on ZNIC-2M \cite{znic} and finetune on eight multi-task binary classification datasets \cite{znicdown}. 
The baselines include six of nine methods the same as the ones in unsupervised learning and two different pre-train strategies in \cite{gnnpretrain}, i.e., attribute masking and context prediction.
Furthermore, we evaluated the transferability of our approach on the PPI-306k \cite{ppi} dataset. 
The details of datasets are in \textbf{Appendix} \ref{app:dataset}. 
The evaluate protocols and the model architectures are summarized in \textbf{Appendix} \ref{app:evaluate protocols} and \textbf{Appendix} \ref{app:model configurations}. 
The hyper-parameter analysis is in \textbf{Appendix} \ref{app:sensitivity analysis}.


\subsection{Unsupervised Learning}
The results of unsupervised graph-level representations for downstream graph classification tasks are shown in Table \ref{tab:unsupervised learning}. Our method consistently ranks among the top 3 and achieves the lowest average rank of 2.8, outperforming the SR-based method RGCL and other methods without rationalizations. 
The findings demonstrate the capability of our method to learn discriminative representations.

\subsection{Transfer Learning}
The results of transfer learning on ZNIC-2M are presented in Table \ref{tab:transfer learning}. 
By utilizing DR to construct embeddings that preserve semantic information, our DRGCL framework achieves top-3 performance on six out of eight datasets and exhibits the highest average accuracy compared to existing baselines.
Our method demonstrates superior transferability compared to other baselines, providing empirical evidence that it can learn more essential rationales in graphs.

The transfer learning results on PPI-306K are shown in Table \ref{tab:transfer learning on bio dataset}, where our method shows competitive or better transferability than other pre-training schemes.

\begin{figure*}[t]
    \centering	\includegraphics[width=0.75\textwidth]{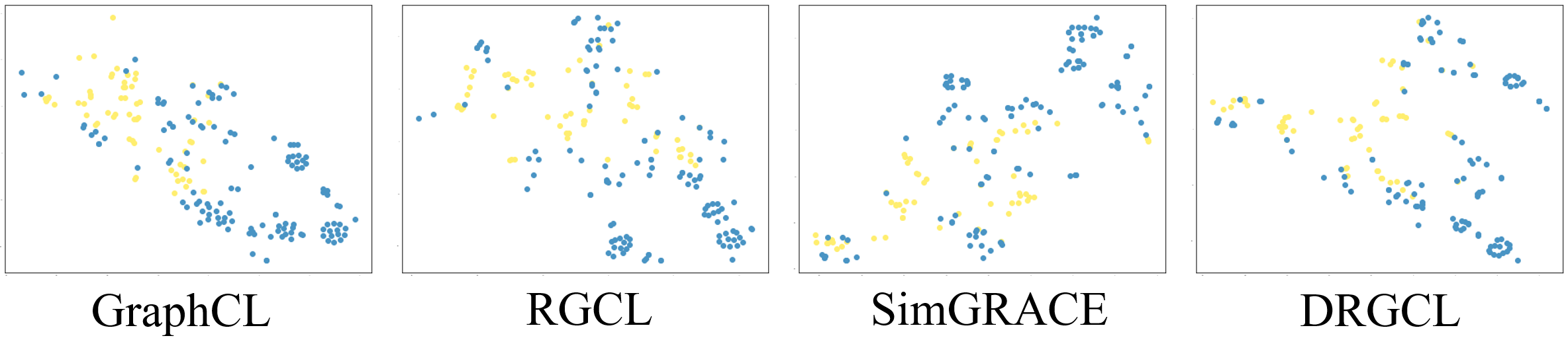}
	\caption{T-SNE visualization of four methods on MUTAG.}
	\label{fig:tsne}
\end{figure*}

\begin{figure}[t]
\begin{center}
\includegraphics[width=0.42\textwidth]{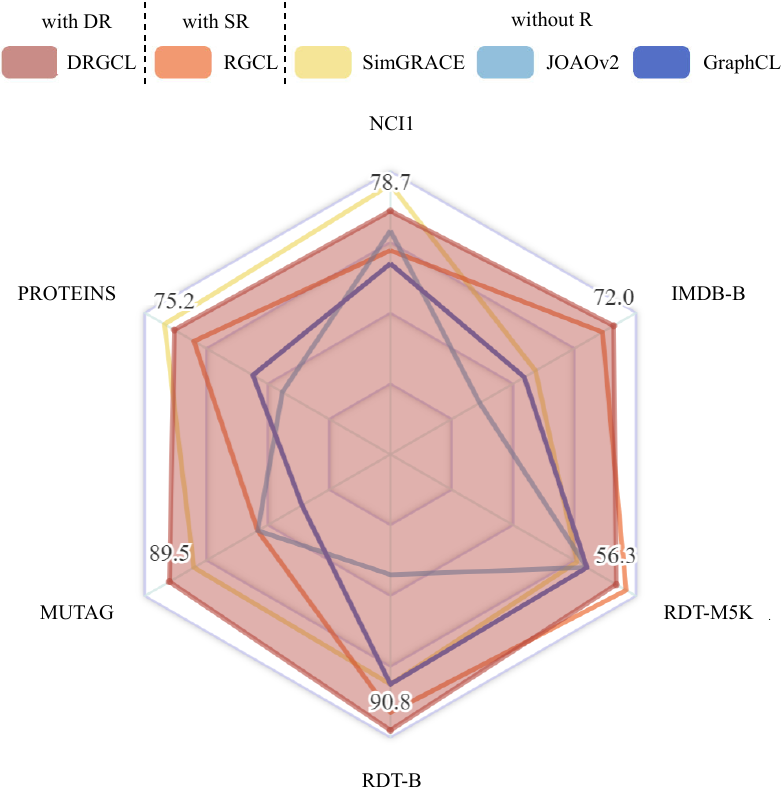} \vspace{-0.05cm}
\caption{Visualization of unsupervised learning results on six data sets for the top-5 methods. \textbf{with DR} denotes our method with DR, \textbf{with SR} denotes the SR method RGCL, and \textbf{without R} denotes the other methods without using rationale.}
\vskip -0.66cm
\label{fig:radar}
\end{center}
\end{figure}

\subsection{Ablation Studies}
We conducted ablation studies in transfer learning, shown in Table \ref{tab:transfer learning ablation}. Our method experiences a decrease of 0.6 in the average score when the DR weight is removed (w/o DR), highlighting the significance of DR. And the decrease of 1.5 in the average score when eliminating the redundancy part (w/o RR) reveals the importance and functionality of RR. It is important to note that the framework GraphCL represents the absence of DR and RR. Notably, both the results of w/o DR and w/o RR outperform GraphCL, emphasizing the positive impact of the two components. 

\subsection{Transfer Learning with Different Fixed $\mathcal{R}$}
In ablation studies, the setting without DR is equivalent to pre-training our model with a fixed dimensional weight, i.e., $\mathcal{R}$, where the meta-learning module is not applied to keep the fixed $\mathcal{R}$ and each dimension of $\mathcal{R}$ is set to 1. In this experiment, we further explore the results of different fixed dimensional weights $\mathcal{R}$ when the DR is not applied. We conducted experiments in transfer learning on ZNIC-2M with fixed $\mathcal{R}$ in $[0.1,0.3,0.7,1.0]$, where each dimension of $\mathcal{R}$ is set the same value.  
The results of transfer learning on ZNIC-2M with different fixed $\mathcal{R}$ are shown in Table \ref{tab:transfer learning with different R}. We notice a consistent phenomenon that our DRGCL method with the DR module which updates $\mathcal{R}$ in a meta-learning manner outperforms the other four methods with different fixed dimensional weights.
This experiment further proves the significance of the DR.

\subsection{Visualization Results}

In Figure \ref{fig:radar}, we visualize the experimental results of the unsupervised learning comparisons.
We plot a radar chart with each direction denoting a dataset, the vertexes of the lines denoting the downstream classification results, and the different colors denoting the top-5 methods of unsupervised learning. Note the scale of each direction is different. 
The visualization results significantly show the performance superiority of the proposed DRGCL over benchmarks. This observation further proves the validity of our findings, i.e., compared with the conventional SR and methods without R, the DR is relatively intrinsic to graphs.

\section{Conclusion} \label{sec:conclusion}
In this paper, we elucidate the causal association among graph embeddings, contrastive labels, and graph DRs, subsequently formulating it through the application of a rigorous SCM.
To eliminate the task-agnostic information during 
pre-training, we propose DRGCL as an intuitive approach to adaptively capture DRs in graph embeddings, which introduces a learnable DR weight updated by a bi-level optimization and a graph DR redundancy reduction regularization term implemented. Benefiting from acquiring DR and reducing the redundancy in graph embeddings, our method achieves new state-of-the-art performance compared to various GCL methods on multiple benchmarks.

\textbf{Limitations and broader impacts}.
Due to the needing for a bi-level optimization, it will cost more time to train a model with the ability to capture DR-aware representations. Besides, our method can be seen as a plug-and-play layer that can be used with any GCL method on any feature-based dataset. Thus, it's worth exploring the combinations of the rationales during different procedures of GCL, which may be a good research direction next.

\section*{Acknowledgements}
The authors would like to thank the editors and reviewers for their valuable comments. This work is supported by the Fundamental Research Program, Grant No. JCKY2022130C020, the National Funding Program for Postdoctoral Researchers, Grant No. GZC20232812, the CAS Project for Young Scientists in Basic Research, Grant No. YSBR-040, the Youth Innovation Promotion Association CAS, No. 2021106, 2022 Special Research Assistant Grant Project, No. E3YD5901, and the China Postdoctoral Science Foundation, No. 2023M743639.

\newpage


\bibliography{aaai24}

\newpage

\appendix

\section{Experimental Details}

\subsection{Datasets}\label{app:dataset}
For unsupervised learning, we benchmark our proposed DRGCL on eight established datasets in TU datasets \cite{tudataset} including four bioinformatics datasets (NCI1, PROTEINS, DD, MUTAG) and four social network networks datasets (COLLAB, RDT-B, RDT-M, IMDB-B). The details of the datasets are shown in Table \ref{tab:unsupervised datasets}. For transfer learning, we first perform pre-training on ZNIC-2M \cite{znic} dataset and PPI-306K \cite{ppi} dataset. Then we finetune the pretrained-ZNIC model on eight benchmark multi-task binary classification datasets (BBBP, Tox21, ToxCast, SIDER, ClinTox, MUV, HIV, BACE) in the biochemistry domain, which are contained in MoleculeNet \cite{znicdown} and finetune the pretrained-PPI model on PPI-306K dataset. Details of ZNIC-2M and PPI-306K datasets are shown in Table \ref{tab:transfer datasets}.

\begin{table}[ht]
\caption{Datasets statistics for unsupervised learning. BM denotes Biochemical Molecules. SN denotes Social Networks. Num. denotes number. Avg. denotes average.}
\label{tab:unsupervised datasets}
\setlength{\tabcolsep}{2pt}
\begin{center}
    \begin{small}
        \begin{tabular}{ccccc}
\toprule
Datasets & Category & Graph Num. & Avg. Node & Avg. Degree \\
\midrule
NCI1 & BM & 4110 & $29.87$ & $1.08$ \\
PROTEINS & BM & 1113 & $39.06$ & $1.86$ \\
DD & BM & 1178 & $284.32$ & $2.51$ \\
MUTAG & BM & 188 & $17.93$ & $1.10$ \\
\midrule COLLAB & SN & 5000 & $74.49$ & $32.99$ \\
RDT-B & SN & 2000 & $429.63$ & $1.15$ \\
RDT-M & SN & 4999 & $508.52$ & $1.16$ \\
IMDB-B & SN & 1000 & $19.77$ & $4.88$ \\
\bottomrule
\end{tabular}
    \end{small}
\end{center}
\end{table}

\begin{table}[ht]
\caption{Datasets statistics for transfer learning. PT denotes pre-training. FT denotes finetuning.}
\label{tab:transfer datasets}
\setlength{\tabcolsep}{2pt}
\centering
\small
\begin{tabular}{cccccc}
\toprule Datasets & Utilization & Graph Num. & Avg. Node & Avg. Degree \\
\midrule ZINC-2M & PT & 2000000 & $26.62$ & $57.72$ \\
PPI-306K & PT \& FT & 306925 & $39.82$ & $729.62$ \\
\midrule BBBP & FT & 2039 & $24.06$ & $51.90$ \\
Tox21 & FT & 7831 & $18.57$ & $38.58$ \\
ToxCast & FT & 8576 & $18.78$ & $38.52$ \\
SIDER & FT & 1427 & $33.64$ & $70.71$ \\
ClinTox & FT & 1477 & $26.15$ & $55.76$ \\
MUV & FT & 93087 & $24.23$ & $52.55$ \\
HIV & FT & 41127 & $25.51$ & $54.93$ \\
BACE & FT & 1513 & $34.08$ & $73.71$ \\
\bottomrule
\end{tabular}
\end{table}

\subsection{Evaluation protocols}\label{app:evaluate protocols}
In accordance with prior research on graph-level self-supervised representation learning \cite{infograph, graphcl}, we assess the discriminability and transferability of the acquired representations in unsupervised and transfer settings. To accomplish this, we utilize the complete datasets to train DRGCL and obtain graph representations using DRs on unsupervised datasets. Subsequently, these representations are employed as input for a downstream SVM classifier with 10-fold cross-validation. We conduct five runs with different seeds on each dataset and report the mean and standard deviation of classification accuracy. For transfer learning, we pre-train our method on the corresponding datasets and repeat the finetuning procedure ten times with various seeds to evaluate the mean and standard deviation of ROC-AUC scores on each downstream dataset. 

\subsection{Model configurations}\label{app:model configurations}
We use the graph isomorphism network (GIN) \cite{gin} as the encoder following \cite{infograph,graphcl} to attain graph representations for unsupervised and transfer learning. The details of our model architectures and corresponding hyper-parameters are summarized in Table \ref{tab:model architectures and hyper-parameters}.

\begin{table}[htp]
\caption{Model architectures and hyper-parameters.}
\label{tab:model architectures and hyper-parameters}
\setlength{\tabcolsep}{1pt}
\begin{center}
    \begin{small}
        \begin{tabular}{ccccc}
\toprule
Experiment & Unsupervised & Transfer\\
 & learning & learning\\
\midrule
Backbone GNN type & GIN & GIN\\
Backbone neuron & [32,32,32] & [300,300,300,300,300] \\
D. R. Gen. neuron & 96 & 300 \\
Projection neuron & [512,512,512] & [300,300]\\
Pooing type & Global add pool & Global mean pool\\
Pre-train $lr$ & 0.01 & 0.001\\
Finetune $lr$ & - & \{0.01,0.001,0.0001\}\\
Temperature $\tau$ & 0.1 & 0.1\\
Traning epochs & 20 & \{60,80,100\}\\
Trade-off parameter $\lambda$ & 0.001 & 0.001\\
Trade-off parameter $\alpha$ & 10 & 10\\
\bottomrule
\end{tabular}
    \end{small}
\end{center}
\end{table}

\subsection{Running Environment}
We have implemented our method using PyTorch 1.10. The results of unsupervised learning and transfer learning were obtained using a single Geforce RTX3090 GPU with 24G of memory. We performed the experiments on Ubuntu 18.04 as our operating system.

\subsection{Model Complexity}
We evaluate our model's complexity on NCI1 from two aspects, shown in Table \ref{tab:model-complexity}. The GPU usages are similar. The time complexity of our method is higher than GraphCL due to the meta-learning stage, while compared to RGCL, our method only requires half the time and achieves better results.

\begin{table}[htp]
\caption{Model complexity.}
\label{tab:model-complexity}
\setlength{\tabcolsep}{7pt}
\begin{center}
    \begin{small}
        \begin{tabular}{cccc}
\toprule
Models & GraphCL & RGCL & DRGCL \\
\midrule
GPU usage & 1652M & 1766M & 1702M \\
Time cost (1 epoch) & 19.42s & 117.04s & 52.43s \\
\bottomrule
\end{tabular}
    \end{small}
\end{center}
\end{table}

\section{Sensitivity Analysis}
\label{app:sensitivity analysis}
\subsection{Batch size $N$}
In this experiment, we evaluate the effect of batch size $N$ on our model performance. Figure \ref{fig:aba batchsize} shows the classification accuracy of our models after training twenty epochs using different batch sizes from 64 to 512 on NCI1 and REDDIT-BINARY. From the line chart, we can observe that our method gets the best result when setting the batch size to 128.
While batch size is less than 128, larger batch size will bring more negative pairs which can 
help promote the convergence of our GCL loss. While batch size is more than 128, too big a batch size will slow down the speed of model convergence, leading the lower accuracy on downstream tasks.

\begin{figure}[ht]
\begin{center}
\includegraphics[width=0.475\textwidth]{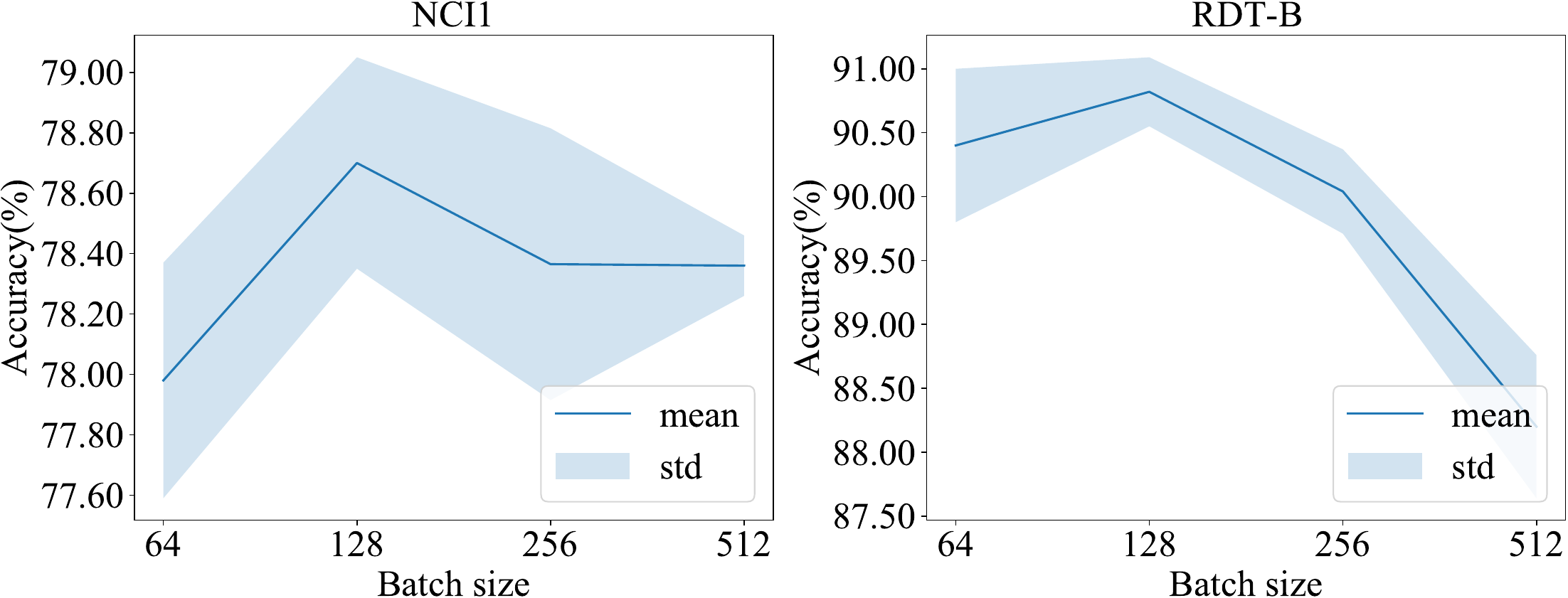}
\caption{Sensitivity analysis for batch size $N$.}
\label{fig:aba batchsize}
\end{center}
\end{figure}

\subsection{Training epochs $T$}
In this experiment, we investigate the impact of training epochs $T$ on the performance of our model. We conduct training sessions with varying numbers of epochs, specifically $\{ 20, 40, 60, 80, 100 \}$, on REDDIT-BINARY. Figure \ref{fig:aba epochs} illustrates the training progress of our method, demonstrating that the accuracy consistently improves with an increase in the number of training epochs.

\begin{figure}[ht]
\begin{center}
\includegraphics[width=0.25\textwidth]{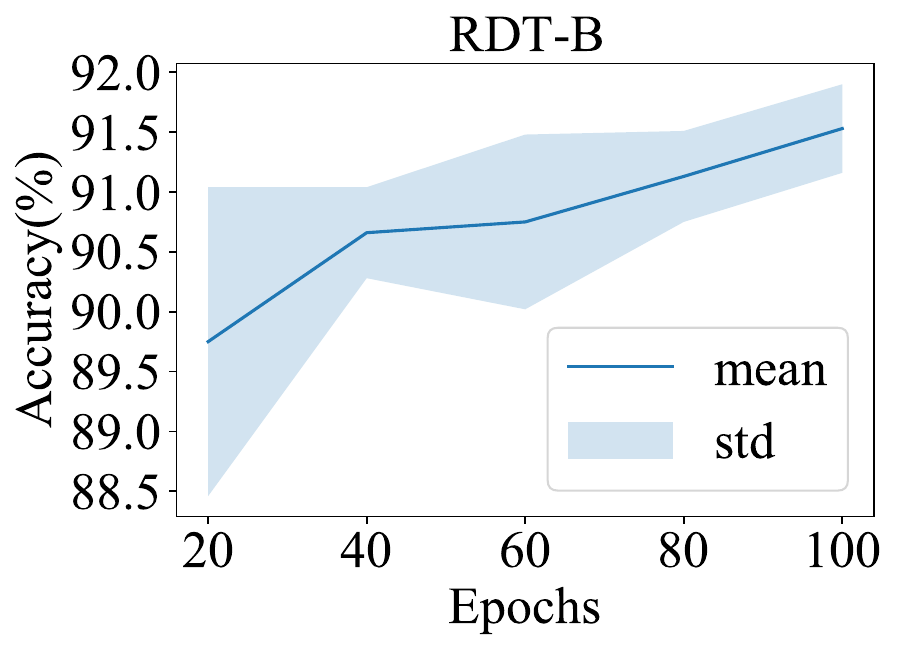}
\caption{Sensitivity analysis for training epochs $T$.}
\label{fig:aba epochs}

\end{center}
\end{figure}

\subsection{Trade-off hyper-parameter $\alpha$}
In this experiment, we investigate the selection of trade-off hyper-parameter $\alpha$ in Equation \ref{eq:all loss}. We conduct experiments under unsupervised settings on NCI1 and RDT-B. The results are shown in Figure \ref{fig:aba alpha}. The comparison results show that an elaborate assignment of $\alpha$ can indeed improve the performance of DRGCL: when $\alpha = 10$ and $100$, the accuracy of DRGCL achieves the peak value, respectively.

\begin{figure}[ht]

\includegraphics[width=0.475\textwidth]{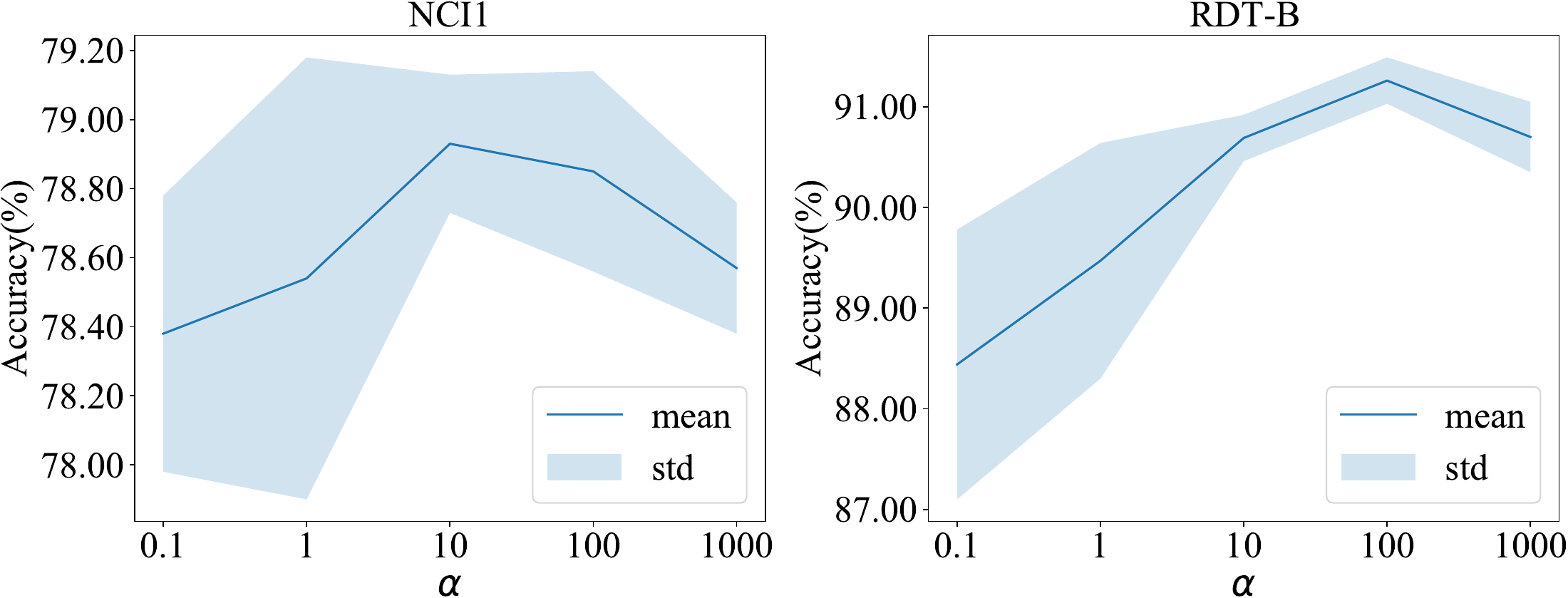}
\caption{Sensitivity analysis for trade-off hyperparameter $\alpha$.}
\label{fig:aba alpha}
\end{figure}

\end{document}